\documentclass[sigconf]{aamas} 

\usepackage{balance}

\usepackage{amssymb}
\usepackage{amsmath}

\usepackage{multirow}
\usepackage[inline, shortlabels]{enumitem}
\setlist{nolistsep, leftmargin=*}

\usepackage[linesnumbered,ruled, noend, noline]{algorithm2e}

\usepackage{macros_amin}

\setcopyright{none} 
\renewcommand\footnotetextcopyrightpermission[1]{} 

\begin{document}

\title{Complexity Results and Algorithms for Bipolar Argumentation}

\author{Amin Karamlou} 
\orcid{0000-0002-7467-090X}
\affiliation{
 \institution{Imperial College London}
  \streetaddress{South Kensington Campus}
  \city{London} 
  \country{United Kingdom}
  \postcode{SW7 2AZ}
}
\email{mak514@imperial.ac.uk} 
\author{Kristijonas \v Cyras}
\orcid{0000-0002-4353-8121}
\affiliation{
 \institution{Imperial College London}
  \streetaddress{South Kensington Campus}
  \city{London} 
  \country{United Kingdom}
  \postcode{SW7 2AZ}
}
\email{k.cyras@imperial.ac.uk}
\author{Francesca Toni}
\orcid{0000-0001-8194-1459}
\affiliation{
  \institution{Imperial College London}
  \streetaddress{South Kensington Campus}
  \city{London} 
  \country{United Kingdom}
  \postcode{SW7 2AZ}
}
\email{f.toni@imperial.ac.uk}
\renewcommand{\shortauthors}{A.~Karamlou et al.}

\begin{abstract}  

Bipolar Argumentation Frameworks (BAFs) admit several interpretations of the support relation and diverging definitions of semantics. 
Recently, several classes of BAFs have been captured as instances of bipolar Assumption-Based Argumentation, 
a class of Assumption-Based Argumentation (ABA). 
In this paper, we establish the complexity of bipolar ABA, 
and consequently of several classes of BAFs. 
In addition to the standard five complexity problems, we analyse the rarely-addressed extension enumeration problem too. 
We also advance backtracking-driven algorithms for enumerating extensions of bipolar ABA frameworks, 
and consequently of BAFs under several interpretations. 
We prove soundness and completeness of our algorithms, describe their implementation and provide a scalability evaluation. 
We thus contribute to the study of the as yet uninvestigated complexity problems of (variously interpreted) BAFs 
as well as of bipolar ABA, 
and provide the lacking implementations thereof. 

\end{abstract}


\maketitle


\section{Introduction}
\label{sec:introduction}

Human-understandable agent interaction is an important topic in multi-agent systems. 
Argumentation has been widely-used to model agent interaction, 
e.g.\ \cite{Parsons:Sierra:Jennings:1998,Ontanon:Plaza:2007,Amgoud:Serrurier:2008,Carrera:Iglesias:2015}, 
especially in the form of debates, 
e.g.\ \cite{Prakken:Sartor:1998,McBurney:Parsons:2009,Fan:Toni:2012-ECAI,DBLP:conf/kr/RagoTAB16}.
Bipolar argumentation 
(see e.g.\ \cite{Cayrol:Lagasquie-Schiex:2013,Cohen:Gottifredi:Garcia:Simari:2014,Cyras:Schulz:Toni:2017}) 
in particular has been shown to be applicable in capturing, formalising and executing debates, 
e.g.\ \cite{DBLP:conf/prima/RagoT17,DBLP:conf/aaai/MeniniCTV18,DBLP:conf/comma/BaroniBRT18}. 
Thus, issues pertaining to practical deployment of bipolar argumentation are of great importance. 

Bipolar Argumentation Frameworks (BAFs) 
(see e.g.\ \cite{Cayrol:Lagasquie-Schiex:2013,Cohen:Gottifredi:Garcia:Simari:2014,Gabbay:2016})
constitute one prominent class of formalisms for bipolar argumentation. 
In particular, they admit various interpretations of support and diverging definitions of semantics,
which arguably impinge their practical deployment. 
Recently, bipolar Assumption-Based Argumentation (\emph{bipolar ABA}) \cite{Cyras:Schulz:Toni:2017} has been shown to subsume BAFs under various interpretations of support 
(to which we henceforth refer to as \emph{various BAFs}), 
thus allowing for the consolidation of theoretical foundations of bipolar argumentation. 
However, the complexity of bipolar ABA and various BAFs is largely unknown 
(except for some complexity problems in one form of BAFs, namely deductive BAFs \cite{Cayrol:Lagasquie-Schiex:2013}, as given in \cite{Fazzinga.et.al:2018}). 
What is more, implementations of bipolar argumentation are generally lacking too 
(except for deductive BAFs, as given in \cite{DBLP:journals/argcom/EglyGW10}). 
This is despite the fact that computational problems in bipolar argumentation have tremendous potential for practical use, 
for instance, in knowing the effectiveness of answering questions such as ``Does there exist a winner of the debate?'', 
or in yielding all the `winning' arguments. 

In this paper, we address the above issues and provide complexity results as well as implementations for bipolar ABA, and therefore indirectly for various BAFs. 
Specifically, we analyse the (non-empty) existence, verification, (credulous and sceptical) acceptance and enumeration complexity problems in bipolar ABA under the semantics capturing various BAFs. 
We establish that bipolar ABA is equally as complex as abstract argumentation (AA) \cite{Dung:1995}. 
We then give algorithms for extension enumeration in bipolar ABA, 
which effectively capture solutions to other complexity problems too. 
We describe an implementation of bipolar ABA as well as various BAFs and complement it with a scalability evaluation, 
showing that our system is fit for practical deployment.

The paper is organised as follows. 
Section \ref{sec:background} provides background on argumentation and complexity theory, as well as existing complexity results for AA. 
We give the complexity results for bipolar ABA in Section \ref{sec:complexityResults}. 
In Section \ref{sec:algorithms} we advance new algorithms for implementing bipolar argumentation. 
We describe the software system implementing these algorithms in Section \ref{sec:implementation}, 
alongside evaluating the system's scalability practically. 
We review related work in Section \ref{sec:related} 
and discuss conclusions and future work in Section \ref{sec:conclusions}.


\section{Background}
\label{sec:background}

We here give background on argumentation and complexity.

\subsection{Argumentation}
\label{subsec:argumentation}

We start with background on argumentation.

\subsubsection{Assumption-Based Argumentation (ABA)}
\label{subsubsec:ABA}

Background on ABA and its restriction Bipolar ABA follows \cite{Bondarenko:Dung:Kowalski:Toni:1997,Toni:2014,Cyras:Schulz:Toni:2017,Cyras:Fan:Schulz:Toni:2018}.

\label{definition:ABA framework}
An \emph{ABA framework} is a tuple $\abaf$, where:
\begin{itemize}
\item $(\LL, \R)$ is a deductive system with $\LL$
a language (i.e.~a set of sentences)  
and $\R$ a set of rules of the form $\varphi_0 \leftarrow \varphi_1, \ldots, \varphi_m$ 
with $m \geqslant 0$ and $\varphi_i \in \LL$ for $i \in \{ 0, \ldots, m \}$; 
$\varphi_0$ is the \emph{head} and
$\varphi_1, \ldots, \varphi_m$ the \emph{body};
if $m = 0$, then $\varphi_0 \leftarrow \varphi_1, \ldots, \varphi_m$ 
has an empty body, and is written as 
$\varphi_0 \leftarrow \top$, where $\top \not\in \LL$;
\item $\A \subseteq \LL$ is a non-empty set of \emph{assumptions};
\item $\contrary : \A \to \LL$ is a total map: 
for $\asma \in \A$, the $\LL$-sentence $\contr{\asma}$ is referred to as the \emph{contrary} of $\asma$.
\end{itemize}

For the remainder of this section, 
we assume as given a fixed but otherwise arbitrary ABA framework $\F = \abaf$. 

\begin{itemize}
\item 
\label{definition:deduction}
A \emph{deduction for $\varphi \in \LL$ supported by $A \subseteq \A$ and $R \subseteq \R$}, 
denoted $A \vdash^R \varphi$, 
is a finite tree with: 
the root labelled by $\varphi$; 
leaves labelled by $\top$ or assumptions, with $A$ being the set of all such assumptions; 
the children of non-leaves $\psi$ labelled by the elements of the body of some 
$\psi$-headed rule in $\R$, 
with $R$ being the set of all such rules. 

\item 
\label{definition:assumption-level attack}
 $\asmA \subseteq \A$ \emph{attacks} $\asmB \subseteq \A$,
denoted $\abaattack{\asmA}{\asmB}$,
if there is a deduction 
$\deduction{\asmA'}{\contr{\asmb}}$
such that $\asmb \in \asmB$, 
$\asmA' \subseteq \asmA$ and $R \subseteq \R$. If it is not the case that $\asmA$ attacks $\asmB$, we may write $\nabaattack{\asmA}{\asmB}$. 

\end{itemize}
Let $\asmA \subseteq \A$:
\begin{enumerate*}
\label{definition:closure}
\item The \emph{closure} of $\asmA$ is 
$\cl(\asmA) = \{ \asma \in \A~:~\exists A' \vdash^R \asma,$ $\ A' \subseteq \asmA,\ R \subseteq \R \}$.
\item $\asmA$ is \emph{closed} iff $\asmA = \cl(\asmA)$. 
\item \F\ is \emph{flat} iff every $\asmA \subseteq \A$ is closed. 
\item $\asmA$ is \emph{conflict-free} iff $\nabaattack{\asmA}{\asmA}$.
\item $\asmA$ \emph{defends} $\asma \in \A$ iff for all closed $\asmB \subseteq \A$ with $ \abaattack{\asmB}{\{ \asma \}}$ 
it holds that $\abaattack{\asmA}{\asmB}$. We also say $\asmA$ defends $\asmB \subseteq \A$ if $\asmA$ defends every $\asmb \in \asmB$.
\end{enumerate*}

We use the following ABA semantics. 
A set $\asmE \subseteq \A$, also called an \emph{extension}, is: 
\label{definition:ABA semantics} 
\begin{enumerate*}
\item \emph{admissible} iff it is closed, conflict-free and defends itself. 
\item \emph{preferred} iff it is $\subseteq$-maximally admissible.
\item \emph{stable} iff it is closed, conflict-free
and $\abaattack{\asmE}{\{ \asma \}} ~~ \forall \asma \in \A \setminus \asmE$.
\item \emph{set-stable} iff it is closed, conflict-free and 
$\abaattack{\asmE}{\cl(\{ \asma \})} ~~ \forall \asma \in \A \setminus \asmE$.
\end{enumerate*}

In $\F$, 
a stable extension is set-stable, 
a set-stable extension is preferred, 
and if $\F$ is flat, then a set-stable extension is also stable. 

The restricted class Bipolar ABA is defined thus. 
\label{definition:bipolar ABA}
An ABA framework $\abaf$ is \emph{bipolar} iff every rule in \R\ is of the form $\varphi \ot \asma$,  
where $\asma \in \A$ and 
either $\varphi \in \A$ or $\varphi = \contr{\asmb}$ for some $\asmb \in \A$. 

Bipolar (just as flat) ABA frameworks admit admissible and preferred 
but not, in general, stable or set-stable extensions.

\subsubsection{Abstract Argumentation (AA)}
\label{subsubsec:AA}

We give background on AA following \cite{Dung:1995}. 
An \emph{AA framework} (AF) is a pair $\DAF$ with a (finite) set $\Args$ of arguments and a binary attack relation $\defeats$ on $\Args$. 
Notions of conflict-freeness and defence, as well as semantics of admissible, preferred and stable extensions are defined verbatim as for ABA, 
but with (sets of) arguments replacing (sets of) assumptions and the closure condition dropped. 
(As in flat ABA, set-stable and stable semantics coincide.)

\subsection{Elements of Complexity}
\label{subsec:complexity}

We assume knowledge of fundamental time and space complexity classes, as well as the concepts of hardness and completeness \cite{Papadimitriou:1994}. 
Thus, we here recap the complexity problems studied in argumentation, as well as established results for AFs.

\subsubsection{Enumeration}
\label{subsubsec:enumeration}

We first give (the less standard) enumeration problems and related complexity classes following \cite{Kroll:Pichler:Woltran:2017}.
    An \emph{enumeration problem} is a pair $(L, Sol)$ such that $L \subseteq \Sigma^*$ (for an alphabet $\Sigma$ containing at least two symbols) and $Sol: \Sigma^* \rightarrow 2^{\Sigma^*}$ is a function such that for all $x \in \Sigma^*$, we have that the set of \emph{solutions} $Sol(x)$ is finite, and $Sol(x) = \emptyset$ iff $x \notin L$.
    An \emph{enumeration algorithm} $\mathcal{A}$ for an enumeration problem $\mathcal{P} = (L, Sol)$ 
    outputs, on input $x$, exactly the elements from $Sol(x)$ without duplicates.
For enumeration algorithms, we use the RAM model of computation \cite{Kroll:Pichler:Woltran:2017}.

The complexity classes $\mathsf{OutputP}$ and \nOP\ are defined thus.
Let $\mathcal{P} = (L, Sol)$ be an enumeration problem. 
$\mathcal{P} \in \mathsf{OutputP}$ if there exists an enumeration algorithm $\mathcal{A}$ for $\mathcal{P}$ and some $m \in \mathbb{N}$, such that on every input $x$, algorithm $\mathcal{A}$ terminates in time $\mathcal{O}((|x| + |Sol(x)|)^m)$.  
Problems \emph{not} in $\mathsf{OutputP}$ constitute the class $\mathsf{nOP}$.

The following decision problem -- \complexityProblem[][][$\mathcal{P}$]{MANYSOL}: 
Given $x \in L$ and a positive integer $m$ in unary notation, is $|Sol(x)| \geq m$? -- is strongly related to the enumeration problem $\mathcal{P} = (L, Sol)$: 
If \complexityProblem[][][$\mathcal{P}$]{MANYSOL} $\notin \P$, then $\mathcal{P} \notin \mathsf{OutputP}$. 
We will use \textbf{MANYSOL} in our analysis of the enumeration problem in bipolar ABA.

\subsubsection{Problems of Interest}
\label{subsubsec:problems}

We now state the problems we are interested in. 
In the following, $\F$ stands for a bipolar ABA framework 
and $\sigma \in \{$\semantics$\}$ denotes a semantics, 
where \adm, \prf and \setstb\ abbreviate 
admissible, preferred and set-stable, respectively. 

\label{definition:standard-problems}
\begin{enumerate}[label=\bf{\arabic*.}]
\item \textsc{Existence (\complexityProblem[$\sigma$][$\F$]{EX}):} 
Does $\F$ admit a $\sigma$ extension?

\item \textsc{Non-Empty Existence (\complexityProblem[$\sigma$][$\F$]{NE})}: 
Does $\F$ admit a non-empty $\sigma$ extension?

\item \textsc{Verification (\complexityProblem[$\sigma$][$\F$][$A$]{VER}):} 
Given $A \subseteq \A$, is $A$ a $\sigma$ extension of $\F$?

\item \textsc{Credulous Acceptance (\complexityProblem[$\sigma$][$\F$][$a$]{CA}):} 
Given $a\in \LL$, is there a $\sigma$ extension $A$ of $\F$ such that 
$\deduction{\asmA'}{a}$ for some $A' \subseteq A$ and $R \subseteq \R$?

\item \textsc{sceptical Acceptance(\complexityProblem[$\sigma$][$\F$][$a$]{SA}):} 
Given $a \in \LL$, 
is it the case that for every $\sigma$ extension $A$ of $\F$ it holds that 
$\deduction{\asmA'}{a}$ for some $A' \subseteq A$ and $R \subseteq \R$?

\item \textsc{Extension Enumeration(\complexityProblem[$\sigma$][$\F$]{EE}):}
Return all $\sigma$ extensions of $\F$.
\end{enumerate}

The above complexity problems admit natural counterparts in BAFs (as well as AFs). 
In fact, the only difference is in the credulous and sceptical acceptance problems, 
for which instead of asking for deductions as in bipolar ABA, 
one asks for containment in extensions in B(AFs), 
see e.g.\ \cite{Dunne:Wooldridge:2009,Dvorak:Dunne:2017,Fazzinga.et.al:2018}. 
As various BAFs are captured in bipolar ABA via a polynomial mapping \cite{Cyras:Schulz:Toni:2017}, 
our complexity results for bipolar ABA in this paper will cover various BAFs too. 

Existing complexity results for AFs are summarised in Table \ref{complexity-summary-AFs} 
(\stb\ stands for stable); 
see \cite{Dvorak:Dunne:2017,Dunne:Wooldridge:2009} for surveys of these results. 

\begin{table}[h]
\caption{Existing complexity results for AFs. 
(Here and henceforth, Y stands for `Yes' and N stands for `No'.}
\label{complexity-summary-AFs}
\makebox[\linewidth]{\begin{tabular}{| c | c | c | c | c | c | c |}
\hline
\textbf{\sem} & \textbf{Ex} & \textbf{NE} &\textbf{VER} & \textbf{CA} & \textbf{SA} & \textbf{EE} \\ \hline
\adm & Trivial (Y)& \NP-c & $\P$ & \NP-c & Trivial (N) & $\mathsf{nOP}$ \\ \hline
\prf & Trivial (Y)& \NP-c & \coNP-c & \NP-c & $\mathsf{\Pi_2^P}$-c  & $\mathsf{nOP}$ \\ \hline
\stb & \NP-c & \NP-c & $\P$ & \NP-c & \coNP-c & $\mathsf{nOP}$ \\ \hline
\end{tabular}}
\end{table}


\section{Complexity Results}
\label{sec:complexityResults}

In this section, we prove new complexity results for the complexity problems in bipolar ABA. 
Table \ref{complexity-summary} summarises our results.

\begin{table}[h]
\caption{Complexity results for bipolar ABA frameworks.}
\label{complexity-summary}
\makebox[\linewidth]{\begin{tabular}{ | c | c | c | c | c | c | c |}
\hline
\textbf{\sem} & \textbf{Ex} & \textbf{NE} &\textbf{VER} & \textbf{CA} & \textbf{SA} & \textbf{EE} \\ \hline
\adm & Trivial (Y)& \NP-c & $\P$ & \NP-c & Trivial (N) & $\mathsf{nOP}$ \\ \hline
\prf & Trivial (Y)& \NP-c & \coNP-c & \NP-c & $\mathsf{\Pi_2^P}$-c  & $\mathsf{nOP}$ \\ \hline
\setstb & \NP-c & \NP-c & $\P$ & \NP-c & \coNP-c & $\mathsf{nOP}$ \\ \hline
\end{tabular}}
\end{table}

These results show that the problems for bipolar ABA frameworks belong to precisely the same complexity classes as their corresponding problems for AFs. 
As a consequence, the same results apply to the various BAFs investigated in \cite{Cyras:Schulz:Toni:2017}.

We first present prerequisite results needed for all of the problems, 
then we study verification, before moving on to existence, acceptance, and enumeration problems.
Note that because there exists a polynomial time mapping between AFs and bipolar ABA frameworks \cite{Cyras:Schulz:Toni:2017}, all the computational problems for bipolar ABA are at least as hard as their AF counterparts.

Throughout, unless stated otherwise, we assume as given a fixed but otherwise arbitrary bipolar ABA framework $\F = \abaf$.

\subsection{Prerequisite Results} 
\label{subsec:prerequisite}

The derivability problem for ABA frameworks is as follows.

\begin{description}
    \item \textsc{Derivability(\complexityProblem[][$\mathcal{F}$][$A, \alpha$]{DER})}: 
Given a set $A \subseteq \mathcal{A}$ and $\alpha \in \mathcal{L}$, does there exist a deduction of the form $\deduction{A}{\alpha}$?
\end{description}

\begin{proposition} \label{DER-NL-complete}
\complexityProblem[][$\mathcal{F}$][$A,\alpha$]{DER} is \NL-complete (thus in \P).
\end{proposition}
\begin{proof}
\item \emph{Membership.} The following algorithm operates in logspace and nondeterministically solves the \complexityProblem[][$\mathcal{F}$][$A,\alpha$]{DER} problem:
\begin{enumerate*}
\item Create a variable $\beta$.
\item Set $\beta$ equal to an arbitrary element of $A$. If $\beta = \alpha$, output `yes'. Otherwise continue.
\item Initiate a counter $k = 0$. 
\item Pick an arbitrary rule $R \in \mathcal{R}$ s.t. $\{\beta\}$ is the body of $R$. If no such rule exists, output `no'. Otherwise continue.
\item If the head of $R$ is equal to $\alpha$, output `yes'. Otherwise continue.
\item set $\beta$ equal to the head of $R$. Increment $k$ by 1. If $ k \geq |\mathcal{A}|$, output `no'. Otherwise return to step 4.
\end{enumerate*}
Note that this algorithm 
operates in log space since the space usage of counter $k \leq \log({\mathcal{|R|}})$.
\item \emph{Hardness.} We provide a (logspace) reduction from \textsc{Reachability}, the canonical \NL-complete problem \cite{Papadimitriou:1994}. 
\begin{description}
\item \textsc{Reachability} (\complexityProblem[][][$G,s,t$]{RCH}): Given a directed graph $G$ and vertices $s$ and $t$ of $G$, is there a path from $s$ to $t$ in $G$?
\end{description}
The mapping below transforms a directed graph $G$ into a language $\mathcal{L}$ and a set of bipolar ABA rules $\mathcal{R}$:
\begin{itemize}
\item $\mathcal{L} = \A = \{x: x$ is a node of $G\}$,
\item $\mathcal{R} = \{y \leftarrow x: x$ and $y$ are nodes of $G$ and there exists an edge from $x$ to $y$ in $G\}$, 
\item $\contr{\alpha} = \alpha$ for $\alpha \in \A$.
\end{itemize}

This is a logspace transformation since we only need two counters to track the node and edge being considered at any point. Moreover, there is a path from $s$ to $t$ in $G$ iff $s = t$ or there is a chain of rules $R \subseteq \mathcal{R}$ of the form $t \leftarrow \gamma_n \leftarrow \ldots \leftarrow \gamma_2 \leftarrow \gamma_1 \leftarrow s$. 
This is precisely the condition in which \complexityProblem[][$\mathcal{F}$][$\{s\},t$]{DER} would output `yes'. 
Thus \complexityProblem[][][$G,s,t$]{RCH} is logspace reducible to \complexityProblem[][$\mathcal{F}$][$A,\alpha$]{DER}. This means that \complexityProblem[][$\mathcal{F}$][$A,\alpha$]{DER} is \NL-hard.
\end{proof}

We now analyse the fundamental properties of conflict-freeness and closure pertaining to all semantics considered in this paper. 

\begin{description}
\item \textsc{Conflict-Freeness} (\complexityProblem[][$\mathcal{F}$][$A$]{CF}): 
Given $A \subseteq \mathcal{A}$, is $A$ conflict-free in $\F$?
\item \textsc{Closure} (\complexityProblem[][$\mathcal{F}$][$A$]{CL}): 
Given $A \subseteq \mathcal{A}$, is $A$ closed in $\F$?
\end{description}

\begin{proposition}\label{prop:CF-CL-in-P}
\complexityProblem[][$\mathcal{F}$][$A$]{CF} and \complexityProblem[][$\mathcal{F}$][$A$]{CL} are in $\P$. 
\end{proposition}

\begin{proof} 
We present \P-time algorithms for both problems.

\emph{\complexityProblem[][$\mathcal{F}$][$A$]{CF}}: For each $\alpha \in A$, use an $\NL$ oracle for \complexityProblem[][$\mathcal{F}$][$A, \overline{\alpha}]{DER}$ to check if $\abaattack{A}{\{\alpha\}}$. If it does, output `no'. Else, output `yes'.
    
\emph{\complexityProblem[][$\mathcal{F}$][$A$]{CL}}: For each $\alpha \notin A$, use an $\NL$ oracle for \complexityProblem[][$\mathcal{F}$][$A,\alpha$]{DER} to check if $\deduction{A}{\alpha}$ for some $R \subseteq \mathcal{R}$. If it does, output `no'. Otherwise, output `yes'.
\end{proof}

We will use the following result, 
which says that in a bipolar ABA framework, no sentence is deducible without any assumptions.

\begin{lemma}\label{lemma-no-empty-deductions-in-bipolar ABA}
There is no deduction in $\F$ of the form $\deduction{\emptyset}{\varphi}$ for any $\varphi \in \mathcal{L}$ and $R \subseteq \mathcal{R}$.
\end{lemma}

\begin{proof}
Bipolar ABA frameworks do not contain any facts (i.e.\ rules of the form $\alpha \leftarrow \top$). 
Hence, it is impossible to have $\top$ as the child of any node in a bipolar ABA deduction. As a result, a deduction of the form $\deduction{\emptyset}{\overline{\varphi}}$ does not exist.
\end{proof}

\subsection{Verification}

We now analyse the complexity of the verification problem under the admissible, preferred and set-stable semantics. In order to prove the results for admissible semantics, we first introduce the notion of minimal attacks in ABA.

\begin{definition}
$\assumptionsubseteq{A}$ \textbf{minimally attacks} $\assumptionsubseteq{B}$, 
denoted by $\minabaattack{A}{B}$, 
iff $\abaattack{A}{B}$ and there is no $A' \subset A$ s.t. $\abaattack{A'}{B}$. 
\end{definition}
\begin{lemma}
\label{min-attack-lemma}
All minimal attacks are of the form $\minabaattack{\{\alpha\}}{B}$ where $\alpha \in \mathcal{A}$ and $B \subseteq \mathcal{A}$.
\end{lemma}
\begin{proof}
Assume there are $\assumptionsubseteq{A}$ and $\assumptionsubseteq{B}$ s.t. $\minabaattack{A}{B}$ and $|A| \neq 1$. 
Then we have two cases: 
\begin{enumerate*}
\item $|A| = \emptyset$: Lemma \ref{lemma-no-empty-deductions-in-bipolar ABA} implies that $\nabaattack{A}{B}$. This contradicts $\minabaattack {A}{B}$.
\item $|A| > 1$: In order for $\abaattack{A}{B}$ there must exist $\alpha \in A$ and $\beta \in B$ where either  $\alpha = \overline{\beta}$ or there exists a chain of rules $\overline{\beta} \leftarrow \gamma_n \leftarrow \ldots \leftarrow \gamma_2 \leftarrow \gamma_1 \leftarrow \alpha$. 
In both cases we have $\abaattack{\{\alpha\}}{B}$. However, $\{\alpha\} \subset A$ so we have a contradiction to the definition of minimal attacks.  
\end{enumerate*}
In any event, $A = \{\alpha\}$ where $\alpha \in \mathcal{A}$ as required.
\end{proof}

\begin{proposition}\label{bipolar ABA-VER_adm}
\complexityProblem[adm][$\mathcal{F}$][$A$]{VER} is in $\P$.
\end{proposition}

\begin{proof}
\begin{enumerate*}
\item Use a $\P$ oracle for \complexityProblem[][$\mathcal{F}$][$A$]{CF} to check if $A$ is conflict-free. If it is not, output `no'. Otherwise continue.
\item Use a $\P$ oracle for \complexityProblem[][$\mathcal{F}$][$A$]{CL} to check if $A$ is closed. If it is not output `no'. Otherwise continue.
\item For each assumption $\beta \notin A$, call an $\NL$ oracle for \complexityProblem[][$\mathcal{F}$][$A, \overline\alpha$]{DER} $|A|$ times (once for each $\alpha \in A$) to check if $\abaattack{\{\beta\}}{A}$. If it does and $\nabaattack{A}{\{\beta\}}$ output `no'. Otherwise continue.
\item  Output `yes'.
\end{enumerate*}

Note that step 3 is sufficient to check that $A$ defends itself. This follows from the fact that if $\abaattack{B}{A}$ then $\minabaattack{\{\beta\}}{A}$ for some $\beta \in B$ (Lemma \ref{min-attack-lemma}). From the definition of attacks, it follows that if $\abaattack{A}{\{\beta\}}$ then $\abaattack{A}{B}$ as well. Moreover, since step 1 of the algorithm checks that $A$ is conflict-free, we know that $\beta \notin A$. So it suffices to prove that A defends itself against singleton sets of assumptions which are not contained within it.
\end{proof}
\begin{proposition} \label{bipolar ABA-VER_pref}
\complexityProblem[prf][$\mathcal{F}$][$A$]{VER} is \coNP-Complete.
\end{proposition}
\begin{proof} \leavevmode
Membership comes from the following non-deterministic, \P-time algorithm, adapted from \cite{Dimopoulos:Nebel:Toni:2002} , which solves the \complexityProblem[prf][$\mathcal{F}$][$A$]{coVER} problem:
\begin{enumerate*}
\item Use a $\P$ oracle for \complexityProblem[adm][$\mathcal{F}$][$A$]{VER} to check if $A$ is admissible. If it is not, output `yes'. Otherwise continue.
\item Guess an assumption set $A' \supset A$.
\item Use a $\P$ oracle for \complexityProblem[adm][$\mathcal{F}$][$A'$]{VER} to check if $A'$ is admissible. If it is, output `yes'. Otherwise output `no'.
\end{enumerate*}
\end{proof}

\begin{proposition}\label{ABA-VER_set-stable}
\complexityProblem[set-stb][$\mathcal{F}$][$A$]{VER} is in \P.
\end{proposition}

\begin{proof} We present the following \P-time algorithm: 
\begin{enumerate*}
\item Use a $\P$ oracle for \complexityProblem[][$\mathcal{F}$][$A$]{CF} to check if $A$ is conflict-free. 
If it is not, output `no'. Else continue.
\item Use a $\P$ oracle for \complexityProblem[][$\mathcal{F}$][$A$]{CL} to check if $A$ is closed. 
If it is not, output `no'. Else continue.
\item For each $\alpha \notin A$, calculate $Cl(\{\alpha\})$ 
by calling an $\NL$ oracle for \complexityProblem[][$\mathcal{F}$][]{DER} $|\A|$ times, 
and then check if $\abaattack{A}{Cl(\{\alpha\})}$ using an additional $|Cl(\{\alpha\})|$ oracle calls.
If it does not, output `no'. Otherwise, output `yes'.
\end{enumerate*}
\end{proof}

\subsection{Existence and Acceptance}
\label{subsec:existence}

Before proving the remainder of our results we make the following observations. 

\begin{proposition} \label{fact-adm=pref}
\complexityProblem[adm][$\mathcal{F}$]{EX}, \complexityProblem[adm][$\mathcal{F}$]{NE} and \complexityProblem[adm][$\mathcal{F}$]{CA} are respectively equivalent to \complexityProblem[prf][$\mathcal{F}$]{EX}, \complexityProblem[prf][$\mathcal{F}$]{NE} and \complexityProblem[prf][$\mathcal{F}$]{CA}.
\end{proposition}

\begin{proof}
Follows from the fact that every preferred extension is admissible and every admissible extension is a subset of some preferred assumption set \cite[Prop1]{Dimopoulos:Nebel:Toni:2002}.
\end{proof}

\begin{proposition} \label{prop:set-stable-ne-equal-ex}
\complexityProblem[set-stb][$\mathcal{F}$]{EX} is equivalent to \complexityProblem[set-stb][$\mathcal{F}$]{NE}.
\end{proposition}

\begin{proof}
    Assume $\emptyset$ is set-stable in $\mathcal{F}$. 
    As $\mathcal{A} \neq \emptyset$, there is $\alpha \in \mathcal{A}$ s.t.\ $\abaattack{\emptyset}{Cl(\{\alpha\})}$. But this contradicts Lemma \ref{lemma-no-empty-deductions-in-bipolar ABA}. Thus, $\emptyset$ is never set-stable in $\mathcal{F}$, and so existence of a set-stable extension is equivalent to the existence of a non-empty set-stable extension.
\end{proof}

\subsubsection{Existence} 
We now consider (non-empty) existence. 

\begin{proposition}
\label{prop:ex-constant-yes}
\complexityProblem[adm][$\mathcal{F}$]{Ex} and \complexityProblem[prf][$\mathcal{F}$]{Ex} are constant, with answer `yes'.
\end{proposition}
\begin{proof}
$\emptyset$ is conflict-free, defends itself, and, by Lemma \ref{lemma-no-empty-deductions-in-bipolar ABA}, is closed. 
Hence, $\emptyset$ is admissible, which establishes the claim for \complexityProblem[adm][$\mathcal{F}$]{Ex}. 
The claim for \complexityProblem[prf][$\mathcal{F}$]{Ex} then follows from Proposition \ref{fact-adm=pref}.
\end{proof}

Now we switch our attention to the non-emptiness problem. 

\begin{proposition}\label{non-empty-np-complete-for-all}
\complexityProblem[adm][$\mathcal{F}$]{NE}, \complexityProblem[prf][$\mathcal{F}$]{NE}, \complexityProblem[set-stb][$\mathcal{F}$]{NE} and \complexityProblem[set-stb][$\mathcal{F}$]{Ex} are \NP-Complete.
\end{proposition}
\begin{proof} \leavevmode
The following non-deterministic, \P-time algorithm proves membership for admissible and set-stable semantics. 
The results for \complexityProblem[prf][$\mathcal{F}$]{NE} and \complexityProblem[set-stb][$\mathcal{F}$]{Ex} follow from Propositions \ref{fact-adm=pref} and \ref{prop:set-stable-ne-equal-ex}.

\begin{enumerate*}
\item Guess an assumption set $A \subseteq \mathcal{A}$.
\item Use a $\P$ oracle for \complexityProblem[adm][$\mathcal{F}$][$A$]{VER} (or \complexityProblem[set-stb][$\mathcal{F}$][$A$]{VER}) to check if $A$ is an admissible (or set-stable) extension. If not, output `no'. Otherwise Output `yes'.
\end{enumerate*}
\end{proof}

\subsubsection{Credulous and Sceptical Acceptance} 
We now turn to acceptance problems. 

\begin{proposition}
\complexityProblem[adm][$\mathcal{F}$][$\alpha$]{CA}, \complexityProblem[prf][$\mathcal{F}$][$\alpha$]{CA}, and \complexityProblem[set-stb][$\mathcal{F}$][$\alpha$]{CA} are \NP-complete, \complexityProblem[set-stb][$\mathcal{F}$][$\alpha$]{SA} is \coNP-complete, and \complexityProblem[prf][$\mathcal{F}$][$\alpha$]{SA} is $\mathsf{\Pi_2^P}$-complete.
\end{proposition}
\begin{proof} \leavevmode
Membership uses the following algorithm, adapted from \cite{Dimopoulos:Nebel:Toni:2002}, solving \complexityProblem[$\sigma$][$\mathcal{F}$][$\alpha$]{CA} and \complexityProblem[$\sigma$][$\mathcal{F}$][$\alpha$]{coSA}, and our previous results for \complexityProblem[$\sigma$][$\mathcal{F}$][$A$]{VER}. The result for \complexityProblem[prf][$\mathcal{F}$][$\alpha$]{CA} follows from Proposition \ref{fact-adm=pref}.

\begin{enumerate*}
\item Guess $A \subseteq \mathcal{A}$. 
\item Use a \complexityProblem[$\sigma$][$\mathcal{F}$][$A$]{VER} oracle to check if $A$ is a $\sigma$ extension. If it is not, output `no'. Otherwise continue.
\item Use an $\NL$ oracle for \complexityProblem[][$\mathcal{F}$]{DER} to check that the formula under consideration is derivable (or not derivable for \complexityProblem[$\sigma$][$\mathcal{F}$][$\alpha$]{coSA}) from $A$ and $\mathcal{R}$. If it is not, output `no'. 
Otherwise, output `yes'.
\end{enumerate*}
\end{proof}

Sceptical acceptance under admissible semantics is trivial. 

\begin{proposition}
\complexityProblem[adm][$\mathcal{F}$][$\alpha$]{SA} is constant, with answer `no'.
\end{proposition}

\begin{proof}
$\emptyset$ is admissible (as in the proof of Proposition \ref{prop:ex-constant-yes}), 
so any sentence derivable from all admissible extensions is derivable from $\emptyset$. 
However, by Lemma \ref{lemma-no-empty-deductions-in-bipolar ABA}, no such sentence exists. 
\end{proof}

We are left to address the enumeration problem.

\subsection{Extension Enumeration}
\label{subsec:enumeration}

We here establish the complexity of \textbf{EE} in bipolar ABA using the \textbf{MANYSOL} problem (see Section \ref{subsec:complexity}). 

\begin{proposition}
\complexityProblem[adm][$\mathcal{F}$]{EE}, \complexityProblem[prf][$\mathcal{F}$]{EE} and \complexityProblem[set-stb][$\mathcal{F}$]{EE} are in $\mathsf{nOP}$, assuming $\P \neq \NP$.
\end{proposition}

\begin{proof}
\complexityProblem[][][\complexityProblem[$\sigma$][AF]{EE}]{MANYSOL} is $\NP$-hard for $\sigma \in \{\adm, \prf, \stb\}$
    \cite{Kroll:Pichler:Woltran:2017}.  
Because AFs can be mapped into flat bipolar ABA in \P-time \cite{Cyras:Schulz:Toni:2017} 
and since stable and set-stable semantics coincide for flat ABA, 
\complexityProblem[][][\complexityProblem[$\sigma$][$\mathcal{F}$]{EE}]{MANYSOL} is $\NP$-hard for $\sigma \in \{ \semantics \}$.
\end{proof}

This completes the complexity analysis of bipolar ABA as summarised in Table \ref{complexity-summary}.


\section{Algorithms}
\label{sec:algorithms}

We have shown that many of the standard problems for bipolar ABA are non-tractable. 
As such, practical algorithms for solving them must make use of advanced techniques and heuristics. We now propose such algorithms for the \complexityProblem[$\sigma$][$\mathcal{F}$]{EE} problem. 
Note that, effectively, \complexityProblem[$\sigma$][$\mathcal{F}$]{EE} answers the other standard problems too. Having all the $\sigma$ extensions of $\F$ one can establish (non-empty) existence immediately, verification by checking membership in the set of enumerated extensions, and (credulous and sceptical) acceptance by using the efficient algorithm for derivation (cf.\ Proposition \ref{DER-NL-complete}).

The algorithms in this section make use of a backtracking strategy. They recursively traverse a binary tree from left to right, where the root node is the empty set and the tree forks to a left (or right) node by including (or excluding) an assumption. If the current node represents a valid extension, it is added to the solution set. Backtracking occurs whenever the procedure is going down a path which will never lead to a correct solution, at this point, it moves back up the tree and takes a different path instead.

\subsection{Enumeration of Preferred Extensions}

We first give a basic algorithm for enumerating preferred extensions that conveys the main ideas. 

In what follows a labelling is a total mapping $Lab: \mathcal{A} \rightarrow \{$IN, OUT, UNDEC, BLANK, MUST\_OUT$\}$.

\subsubsection{Basic Algorithm}

We define some labellings which correspond to the different states of our algorithm while traversing the binary tree. These definitions and the algorithms following them are inspired by the corresponding work for AFs, particularly \cite{Nofal:Atkinson:Dunne:2016}.

\begin{definition}
A labelling $Lab$ of $\F$ is:
\begin{itemize}
    \item the \textbf{initial labelling} of $\mathcal{F}$ iff  $Lab = 
    \{(\alpha, \text{BLANK}): \alpha \in \mathcal{A} \setminus S\} \cup \{(\beta, \text{UNDEC}): \beta \in S\}$ where $S \subseteq \mathcal{A}$ is the set of all $\gamma \in \mathcal{A}$ s.t.\ $\minabaattack{\{\gamma\}}{Cl(\{\gamma\})}$.
    \item a \textbf{terminal labelling} of $\mathcal{F}$ iff for each $\alpha \in \mathcal{A}, Lab(\alpha) \neq$ BLANK.
    \item a \textbf{hopeless labelling} of $\mathcal{F}$ iff there exists an $\alpha \in$ {MUST\_OUT} s.t. for all $\beta \in \mathcal{A}$,
 if $\minabaattack{\{\beta\}}{\{\alpha\}}$ then $Lab(\beta) \in $ \{OUT, UNDEC\}.
    \item an \textbf{admissible labelling} of $\mathcal{F}$ iff $Lab$ is a terminal labelling of $\mathcal{F}$ and MUST\_OUT = $\emptyset$.
    \item a \textbf{preferred labelling} of $\mathcal{F}$ iff $Lab$ is an admissible labelling of $\mathcal{F}$ and $\{x: Lab(x)=\text{IN}\}$ is maximal (w.r.t.\ $\subseteq$) among all admissible labellings of $\mathcal{F}$.
\end{itemize}
\end{definition}

In the above, the initial labelling corresponds to the root of the binary tree. Terminal labellings correspond to leaf nodes of the tree. If there are no MUST\_OUT assumptions in a terminal labelling, then we have an admissible labelling. Preferred labellings are those which are maximally admissible. Finally, hopeless labellings are those which are guaranteed to not reach an admissible labelling. 

Next, we define two procedures of our algorithm, which correspond to taking the left or right path down our binary tree. 

\begin{definition} \label{defn-ABA-left-transition}
Let $Lab$ be a labelling of $\mathcal{F}$, and $\alpha \in \mathcal{A}$.
\begin{itemize}
    \item The \textbf{left-transition} of $Lab$ to the new labelling $Lab'$ using $\alpha$ is defined by:
    \begin{enumerate*} 
        \item $ Lab' \leftarrow Lab$.
        \item For each $\beta \in Cl(\{\alpha\})$, $ Lab'(\beta) \leftarrow IN$.
        \item For each $\gamma \in \mathcal{A}$, if $\minabaattack{\alpha}{Cl(\{\gamma\})}$, $Lab'(\gamma) \leftarrow$ OUT.
        \item For each $\delta \in \mathcal{A}$, with $Lab(\delta) \neq$ OUT, if $\minabaattack{\delta}{Cl(\{\alpha\})}$, $Lab'(\delta) \leftarrow$ MUST\_OUT.
       \end{enumerate*}
\item The \textbf{right-transition} of $Lab$ to the new labelling $Lab'$ using $\alpha$ is defined by:
    \begin{enumerate*}
        \item $ Lab' \leftarrow Lab$.
        \item For each $\beta$ in $\mathcal{A}$, if $\alpha \in Cl(\{\beta\})$, $ Lab'(\beta) \leftarrow UNDEC$.
    \end{enumerate*}
\end{itemize}
\end{definition}

A left transition starts by labelling all assumptions in the closure of some target assumption as IN. We then label any assumptions whose closure is minimally attacked by the target assumption as OUT. After that, we can add those assumptions which minimally attack the closure of the target assumptions to MUST\_OUT. In a right-transition, we label all assumptions whose closure contains the target assumption as UNDEC.

Algorithm \ref{algorithm:basic-preferred} enumerates all the preferred extensions of $\F$. 

\begin{proposition}
\label{prop:basic}
Algorithm \ref{algorithm:basic-preferred} solves the \complexityProblem[prf][$\mathcal{F}$]{EE} problem. 
\end{proposition}
\begin{proof}[Proof outline]
\leavevmode
\emph{Completeness}. Algorithm \ref{algorithm:basic-preferred} builds every closed, conflict-free subset of $\mathcal{F}$. This is guaranteed by our definitions of initial labelling, left-transition and right-transition.

\emph{Soundness}. We need to show that the generated sets are maximal and admissible. Maximality is ensured by line 4 together with the fact that maximal sets are constructed first (by performing left-transitions before right-transitions). For admissibility we need to show that the sets in $E$ are closed, conflict-free and defend themselves. Closure is guaranteed because as soon as a new assumption is labelled IN, so is every element in its closure. Conflict-freeness is guaranteed since any assumption attacked by the set of IN assumptions is immediately labelled OUT. Defence is guaranteed by our usage of the MUST\_OUT label and hopeless labellings. 
\end{proof}

\begin{algorithm}[h]  \SetKwInOut{Input}{input}  
  \Input{$\genericabaframework$ is a bipolar ABA framework.\\
  \mbox{$Lab:\mathcal{A} \rightarrow \{$IN, OUT, UNDEC, BLANK, MUST\_OUT$\}$}\\
  $E \subseteq 2^\mathcal{A}$
  }
  \SetKwInOut{Output}{output} 
   \Output{    \mbox{$Lab: \mathcal{A} \rightarrow \{$IN, OUT, UNDEC, BLANK, MUST\_OUT$\}$}\\
  $E \subseteq 2^\mathcal{A}$
  }
  \lIf{$Lab$ is a hopeless labelling}  {
      \KwRet
  }
  \If{$Lab$ is a terminal labelling}{
      \If{$Lab$ is an admissible labelling} {
        \If{$\{x: Lab(x) =$ IN$\}$ is not a subset of any set in $E$} {
          $E \leftarrow E \cup \{\{x: Lab(x) =$ IN$\}\}$
        }
      }
      \KwRet\;
  }
  Select any assumption $\alpha \in \mathcal{A}$ with $Lab(\alpha) =$ BLANK\;
  Get a new labelling $Lab'$ by applying the left-transition of $Lab$ using $x$\;
  Call Enumerate\_Preferred$(\mathcal{F}, Lab', E)$\;
  \mbox{Get a new labelling $Lab'$ using the right-transition of $Lab$ with $x$\;}
  Call Enumerate\_Preferred$(\mathcal{F}, Lab', E)$.
  \caption{Enumerate\_Preferred$(\mathcal{F}, Lab, E)$}
   \label{algorithm:basic-preferred}
\end{algorithm}

Figure \ref{fig:labelling-preferred-exampl} shows an example of Algorithm \ref{algorithm:basic-preferred} calculating the preferred extensions of a bipolar ABA framework. The algorithm starts with the initial labelling and forks to the left and right by performing the appropriate transition procedure (represented by left and right arrows in the figure). Leaf nodes are identified as either admissible or preferred extensions (although only preferred ones are saved). Moreover, the figure shows how the notion of hopeless labellings reduces the search-space of the algorithm. 

\begin{figure}[h]
\centering
  \includegraphics[width=0.98\linewidth]{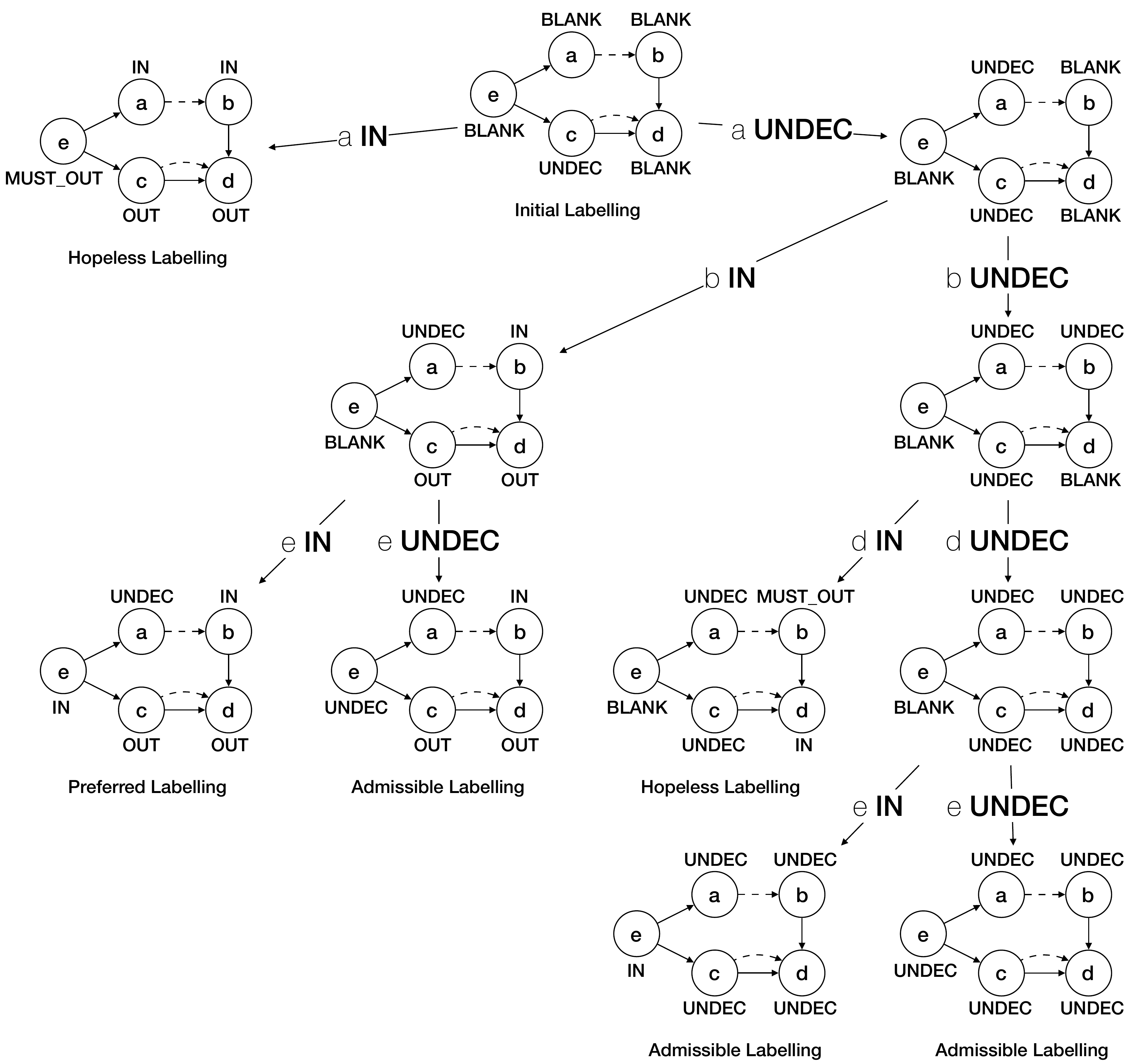}
  \caption{The behaviour of Algorithm \ref{algorithm:basic-preferred} in calculating preferred extensions of a specific bipolar ABA framework. In the above, solid lines correspond to rules of the form $\overline{\alpha} \leftarrow \beta$ and dotted lines to rules of the form $\alpha \leftarrow \beta$, where $\alpha, \beta \in \mathcal{A}$.}
  \label{fig:labelling-preferred-exampl}
\end{figure}

\subsubsection{Improved Algorithm}
We now discuss several improvements of the basic algorithm given above, similarly to \cite{Nofal:Atkinson:Dunne:2016}. 

Algorithm {\ref{algorithm:basic-preferred} can be improved by introducing influential assumptions. The idea is to select the most influential assumption for a left-transition to reach a terminal or hopeless labelling faster.

\begin{definition}
Let $Lab$ be a labelling of $\mathcal{F}$, and $\alpha \in \mathcal(A)$ be such that $Lab(\alpha) =$ BLANK. 
Then $\alpha$ is \textbf{influential} iff for all $\beta \in \mathcal{A}$ with $Lab(\beta) =$ BLANK, $h(\alpha) \geq h(\beta)$ where $h(x)$ is defined as 
the number of rules in $\mathcal{R}$  which contain the assumption $x$ in their head or body.
\end{definition}

 Another improvement comes from realising that assumptions which are minimally attacked only by assumptions labelled OUT or MUST\_OUT have to be labelled IN if the labelling is to evolve into a preferred one. This is because they must be defended by any admissible set reachable from the current labelling.

\begin{definition}
    Let $Lab$ be a labelling of $\mathcal{F}$. Then $\alpha \in \mathcal{A}$ is a \textbf{must\_in assumption} iff $Lab(\alpha) =$ BLANK and for all $\beta \in \mathcal{A}$ where $\minabaattack{\beta}{\{\alpha\}}$, $Lab(\alpha) \in \{$OUT, MUST\_OUT$\}$. The \textbf{labelling propagation} of $Lab$ consists of the following actions:
\begin{enumerate*}
        \item If there is no must\_in assumption, halt.
        \item Pick a must\_in assumption $\alpha$.
        \item Do $Lab(\alpha) \leftarrow$ IN.
        \item For each $\beta \in Cl(\{\alpha\})$, do $ Lab(\beta) \leftarrow IN$.
        \item For each $\gamma \in \mathcal{A}$, if $\minabaattack{\alpha}{Cl(\{\gamma\})}$, do $Lab(\gamma) \leftarrow$ OUT.
        \item For each $\delta \in \mathcal{A}$ with $Lab(\delta) \neq$ OUT, if $\minabaattack{\delta}{Cl(\{\alpha\})}$, do $Lab'(\delta) \leftarrow$ MUST\_OUT.
        \item Return to step 1.
\end{enumerate*}
\end{definition}

We now propose Algorithm \ref{algorithm:optimized-preferred} which adds the following improvements to Algorithm \ref{algorithm:basic-preferred}. 
\begin{enumerate*}
\item The left transition is performed using the most influential assumption.
\item The recursive call after a right transition is replaced with a while loop structure. 
\item Hopeless labellings are checked for every time a labelling changes.
\item  Labelling propagation is added to the start of the algorithm.
 \end{enumerate*}

\begin{algorithm}[h] 
  \SetKwInOut{Input}{input}  
  \Input{$\genericabaframework$ is a bipolar ABA framework.\\
  \mbox{$Lab: \mathcal{A} \rightarrow \{$IN, OUT, UNDEC, BLANK, MUST\_OUT$\}$}\\
  $E \subseteq 2^\mathcal{A}$
  }
  \SetKwInOut{Output}{output} 
   \Output{    \mbox{$Lab: \mathcal{A} \rightarrow \{$IN, OUT, UNDEC, BLANK, MUST\_OUT$\}$}\\
  $E \subseteq 2^\mathcal{A}$
  }
  Propagate $Lab$\;
  \lIf{$Lab$ is a hopeless labelling}  {
      \KwRet
  }
  \While {$Lab$ is not a terminal labelling} {
    Select a new assumption $\alpha \in \mathcal{A}$ s.t. $\alpha$ is influential\;
    Get a new labelling called $Lab'$ by applying the left-transition of $Lab$ using $\alpha$\;    
    \If{$Lab'$ is not a hopeless labelling}  {
      Call Enumerate\_Preferred$(\mathcal{F}, Lab', E)$\;
    }
    Update $Lab$ by applying the right-transition of $Lab$ using $\alpha$\;
    \lIf{$Lab$ is a hopeless labelling}  {
      \KwRet
    }
  }
  \If{$Lab$ is an admissible labelling} {
    \If{$\{x: Lab(x) =$ IN$\}$ is not a subset of any set in $E$} {
      $E \leftarrow E \cup \{\{x: Lab(x) =$ IN$\} \}$.
    }
  }
  \caption{Enumerate\_Preferred$(\mathcal{F}, Lab, E)$}
  \label{algorithm:optimized-preferred}
\end{algorithm}

\begin{proposition}
\label{prop:improved}
Algorithm \ref{algorithm:optimized-preferred} solves the \complexityProblem[prf][$\mathcal{F}$]{EE} problem. \end{proposition}

\begin{proof}[Proof outline]
We show that none of the changes introduced in Algorithm \ref{algorithm:optimized-preferred} compromise soundness or completeness.  
\begin{enumerate*}
\item     Selecting the most influential assumption does not compromise left and right transitions, because by definition this assumption will be labelled BLANK.
\item Changing the right transition to be performed as a while loop doesn't change the order of operations.
\item Checking for hopeless labellings does not have any side effects, so doing it more often will not either.
\item Labelling propagation excludes only admissible labellings which are not preferred. 
 \end{enumerate*}
\end{proof} 

\subsection{Enumeration of Admissible and Set-Stable Extensions}

We next give algorithms for enumerating admissible and set-stable extensions of $\F$. 

\subsubsection{Enumeration of Admissible Extensions}

Algorithm \ref{algorithm:optimized-preferred} can be adapted to find admissible extensions. To achieve this, we first need to drop the maximality check. Moreover, the labelling propagation step needs to be removed. Indeed, if we do not remove it, then there is a risk that some admissible sets will be overlooked since these sets do not necessarily contain every assumption that they defend. 
This modification is achieved in Algorithm \ref{algorithm:admissible-extension-enum}.

\begin{algorithm}[h]   
\SetKwInOut{Input}{input}  
  \Input{$\genericabaframework$ is a bipolar ABA framework.\\
  \mbox{$Lab: \mathcal{A} \rightarrow \{$IN, OUT, UNDEC, BLANK, MUST\_OUT$\}$}\\
  $E \subseteq 2^\mathcal{A}$
  }
  \SetKwInOut{Output}{output} 
   \Output{    \mbox{$Lab: \mathcal{A} \rightarrow \{$IN, OUT, UNDEC, BLANK, MUST\_OUT$\}$}\\
  $E \subseteq 2^\mathcal{A}$
  }
  \lIf{$Lab$ is a hopeless labelling}  {
      \KwRet
  }
  \While {$Lab$ is not a terminal labelling} {
    Select a new assumption $\alpha \in \mathcal{A}$ s.t.\ $\alpha$ is influential\;
    Get a new $Lab$ called $Lab'$ by applying the left-transition of $Lab$ using $\alpha$\;    
    \If{$Lab'$ is not a hopeless labelling}  {
      Call Enumerate\_Admissible$(\mathcal{F}, Lab', E)$\;
    }
    Update $Lab$ by applying the right-transition of $Lab$ using $\alpha$\;
    \lIf{$Lab$ is a hopeless labelling}  {
      \KwRet
    }
  }
  \If{$Lab$ is an admissible labelling} {
    $E \leftarrow E \cup \{\{x: Lab(x) =$ IN$\}\}$
  }
  \caption{Enumerate\_Admissible$(\mathcal{F}, Lab, E)$}
  \label{algorithm:admissible-extension-enum}
\end{algorithm}

Therefore, we have the following result:

\begin{proposition}
\label{prop:admissible}
Algorithm \ref{algorithm:admissible-extension-enum} solves the \complexityProblem[adm][$\mathcal{F}$]{EE} problem.
\end{proposition}

\subsubsection{Enumeration of Set-Stable Extensions}

Algorithm \ref{algorithm:optimized-preferred} can also be adapted to find set-stable extensions. By definition any set-stable extension attacks the closure of all assumption sets it does not contain. Thus, the UNDEC label is no longer useful since any assumption which would have been labelled UNDEC in the case of preferred semantics should now be labelled MUST\_OUT. We change some of our definitions accordingly. 

\begin{definition}
\label{defn:set-stable-lab}
Let $Lab$ be a labelling of $\mathcal{F}$. Then:
\begin{itemize}
    \item $Lab$ is the \textbf{initial set-stable labelling} of $\mathcal{F}$ iff $Lab = \break\{(\alpha, \text{BLANK}): \alpha \in \mathcal{A} \setminus S\} \cup \{(\beta, \text{MUST\_OUT}): \beta \in S\}$ where $S \subseteq \mathcal{A}$ is the set of all $\gamma \in \mathcal{A}$ s.t. $\minabaattack{\{\gamma\}}{Cl(\{\gamma\})}$.
    \item Let $\alpha$ be an assumption in $\mathcal{A}$. Then the \textbf{set-stable right-transition} of $Lab$ to the new labelling $Lab'$ using $\alpha$ is defined by actions:

\begin{enumerate*}
        \item $ Lab' \leftarrow Lab$.
        \item For each $\delta \in Cl(\{\alpha\})$ with $Lab(\delta) \neq$ OUT, $Lab'(\delta) \leftarrow$ MUST\_OUT.
\end{enumerate*}
    \item $Lab$ is a \textbf{set-stable labelling} of $\mathcal{F}$ iff $Lab$ is a terminal labelling of $\mathcal{F}$ and MUST\_OUT = $\emptyset$.
\end{itemize}
\end{definition}

The modifications are achieved in Algorithm \ref{algorithm:set-stable-extension-enum}. 
Thus, as with enumeration of admissible extensions, 
we have the following result:

\begin{proposition}
\label{prop:set-stable}
Algorithm \ref{algorithm:set-stable-extension-enum} solves the \complexityProblem[set-stb][$\mathcal{F}$]{EE} problem.
\end{proposition}

\begin{algorithm}[h] 
  \SetKwInOut{Input}{input}  
  \Input{$\genericabaframework$ is a bipolar ABA framework.\\
  $Lab: \mathcal{A} \rightarrow \{$IN, OUT, BLANK, MUST\_OUT$\}$\\
  $E \subseteq 2^\mathcal{A}$
  }
  \SetKwInOut{Output}{output} 
   \Output{    \mbox{$Lab: \mathcal{A} \rightarrow \{$IN, OUT, BLANK, MUST\_OUT$\}$}\\
  $E \subseteq 2^\mathcal{A}$
  }
  Propagate $Lab$\;
  \lIf{$Lab$ is a hopeless labelling}  {
      \KwRet
  }
  \While {$Lab$ is not a terminal labelling} {
    Select a new assumption $\alpha \in \mathcal{A}$ s.t. $\alpha$ is influential\;
    Get a new labelling called $Lab'$ by applying the left-transition of $Lab$ using $\alpha$\;    
    \If{$Lab'$ is not a hopeless labelling}  {
      Call Enumerate\_Set-stable$(\mathcal{F}, Lab', E)$\;
    }
    Update $Lab$ by applying the right-transition-set-stable of $Lab$ using $\alpha$\;
    \lIf{$Lab$ is a hopeless labelling}  {
      \KwRet
    }
  }
  \If{$Lab$ is a set-stable labelling} {
    $E \leftarrow E \cup \{\{x: Lab(x) =$ IN$\}\}$
  }
  \caption{Enumerate\_Set-stable$(\mathcal{F}, Lab, E)$}
  \label{algorithm:set-stable-extension-enum}
\end{algorithm}

We note that while conceptually similar algorithms exist for enumerating extensions of AFs (see e.g.\ \cite{Nofal:Atkinson:Dunne:2014,Nofal:Atkinson:Dunne:2016, Charwat:Dvorak:Gaggl:Wallner:Woltran:2015}), adapting them to be used for bipolar ABA frameworks is not a trivial task. Specifically, the algorithms described in this section for admissible and preferred semantics are more complex than existing ones for AFs because extensions in bipolar ABA need to be closed, and the notion of defence in ABA is more involved than the corresponding notion of acceptance in abstract argumentation. Moreover, the ideas relating to set-stable labellings (Definition \ref{defn:set-stable-lab}) are new, as is Algorithm \ref{algorithm:set-stable-extension-enum}.


\section{Implementation and Evaluation}
\label{sec:implementation}

We now discuss our implementation of the algorithms mentioned in the previous section. 
In addition to directly enumerating extensions of bipolar ABA frameworks, our system is also capable of calculating extensions of other argumentation frameworks 
(particularly AFs and BAFs (as defined in \cite{Nouioua:Risch:2010,Cayrol:Lagasquie-Schiex:2013,Gabbay:2016}) 
by utilising extension-preserving mappings from these formalisms into bipolar ABA as discussed in \cite{Cyras:Schulz:Toni:2017}. 
This makes our tool more versatile than existing systems \cite{Charwat:Dvorak:Gaggl:Wallner:Woltran:2015}, almost none of which calculate extensions of BAFs.

\subsection{Implementation}
\label{subsec:implementation}

We now describe the control flow of our system as depicted graphically in Figure \ref{fig:control-flow-diagram}.
\begin{enumerate}
\item \textbf{Input argumentation framework.} The user inputs an argumentation framework to the system and specifies which semantics they would like the system to calculate extensions under. The argumentation framework can be an AF, one of the various BAFs, or a bipolar ABA framework. 
\item \textbf{Parse argumentation framework.} The system parses, and generates an internal representation of, the input framework. 
\item \textbf{Perform standard mapping.} If the input framework is not a bipolar ABA framework, the system transforms it into a bipolar ABA framework using the mappings defined in \cite{Cyras:Schulz:Toni:2017}.
\item \textbf{Perform labelling algorithms.} The bipolar ABA framework is inputted to the labelling algorithms defined in section \ref{sec:algorithms}.
\item \textbf{Output Extensions.} Our system terminates after outputting the extensions calculated by the labelling algorithms. 

\end{enumerate} 
\begin{figure}[h]
\centering
  \includegraphics[width=0.75\linewidth]{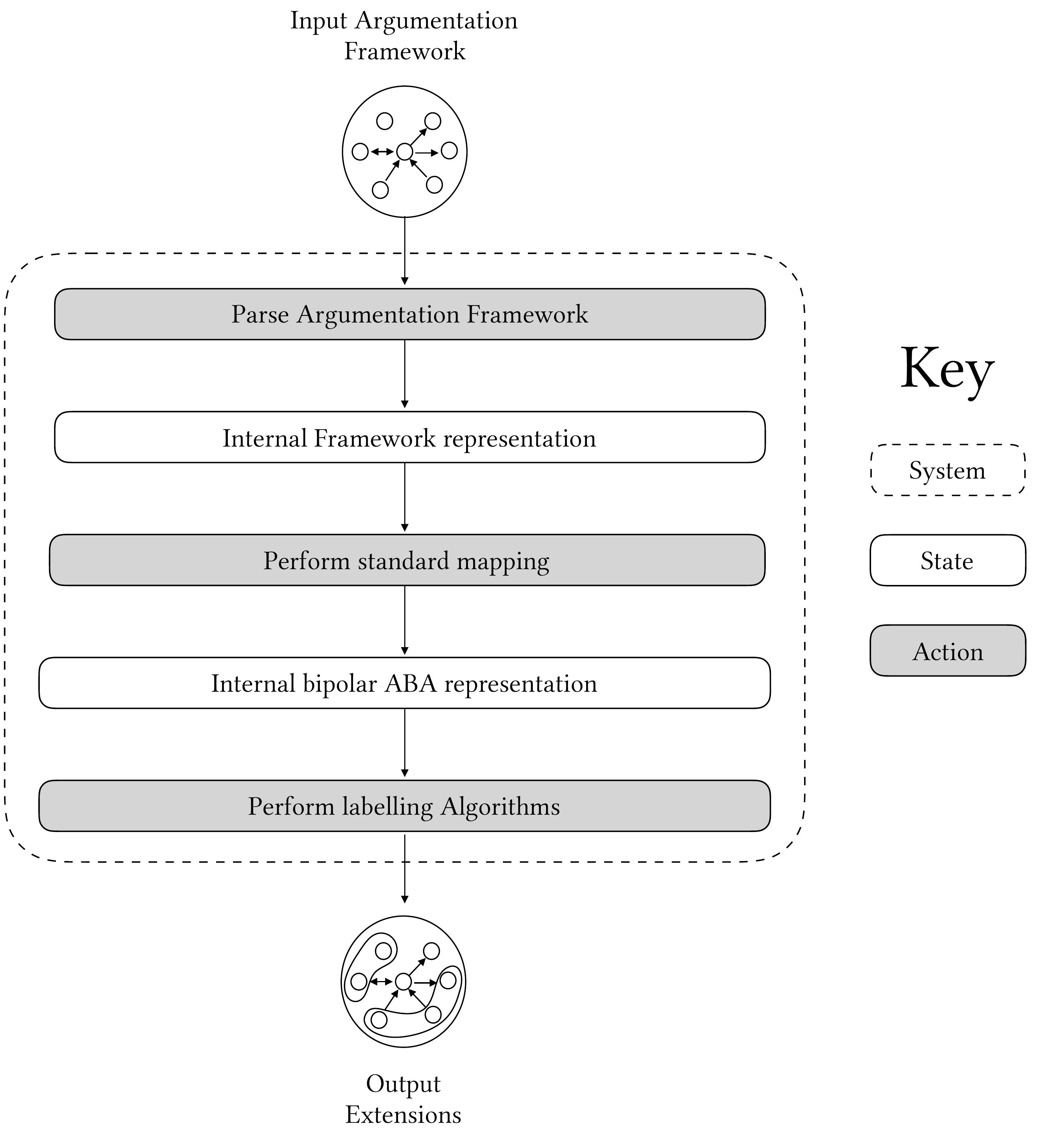}
  \caption{Control flow diagram of the system for computing extensions of bipolar argumentation frameworks.}
  \label{fig:control-flow-diagram}
\end{figure}

\subsubsection{Evaluation} 
\label{subsubsec:evaluation}

In order to test the scalability of our system, we generated 405 bipolar ABA frameworks of increasing size. To do this we adapted an existing benchmark generator from \cite{Craven:Toni:2016}, originally used to create flat ABA frameworks, and ensured that bipolar ABA frameworks are generated instead.
    
We input a tuple of parameters $(N_s, N_a, N_{rh}, N_{rph})$ to the generator in order to create our frameworks. The parameters are defined as follows:
\begin{enumerate*}
    \item $N_s$ is the total number of sentences in the framework, i.e., $|\mathcal{L}|$.
    \item $N_ a$ is the number of assumptions, i.e., $|\mathcal{A}|$, given as a percentage of the number of sentences.
    \item $N_ {rh}$ is the number of distinct sentences to be used as rule heads, given as an integer.
    \item $N_ {rph}$ is the number of rules per distinct rule head, given as an interval [min, max], where min and max are integers.
\end{enumerate*}

The specific parameters used were $(N_s , 37\%, N_s /2, [2, N_s /8]$) with the value of $N_s$ starting at 16, and increasing by 8 between subsequent frameworks. The largest framework consisted of 3248 sentences, 1202 assumptions and 174,365 rules. 

We measured the elapsed time between inputting a framework and outputting its extensions, under the admissible, preferred and set-stable semantics, for all generated frameworks. The elapsed times were very similar for all three semantics. Figure \ref{fig:experiment} shows the time taken to calculate extensions for each framework, averaged over the three semantics. These experiments were run on a home machine, with 16GB of memory and a 2.9GHZ, 2 core CPU. 

\begin{figure}[h]
\centering
  \includegraphics[width=\linewidth]{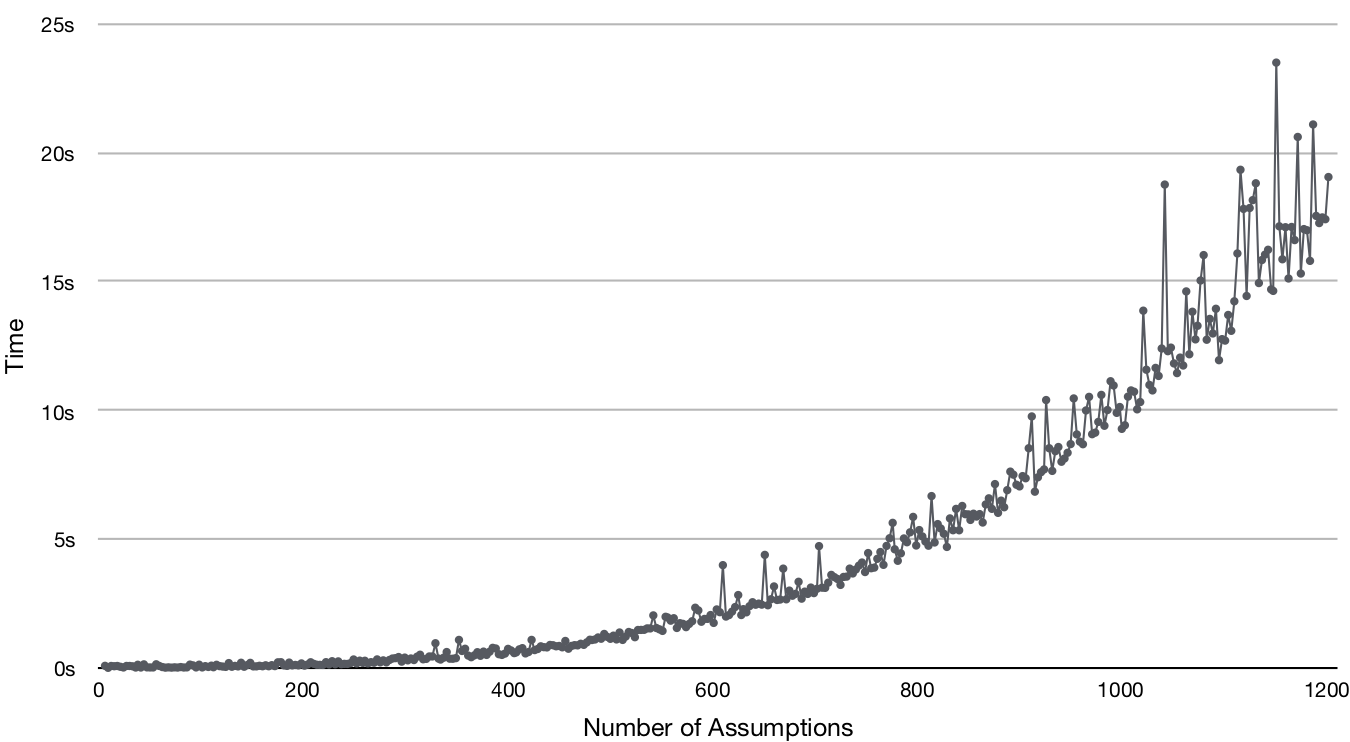}
  \caption{Time taken to calculate extensions of generated bipolar ABA frameworks (averaged over semantics).}
  \label{fig:experiment}
\end{figure}

The results show that even for the largest frameworks considered, our algorithms calculate extensions in under 25 seconds. We do see the performance begin to deteriorate as the size of the frameworks increase. This is expected since the backtracking method we rely on operates in $\mathcal{O}(2^n)$ in the worst case. 
Overall, these results demonstrate the feasibility of our algorithms as a means of generating extensions of large argumentation frameworks.

All in all, we have presented a scalable system for computing and enumerating all extensions of bipolar ABA frameworks under the semantics considered in this paper. 
Consequently, the system computes and enumerates extensions of various formulations of BAFs. 
In addition, it allows to answer questions to all the standard complexity problems considered in this paper.


\section{Related Work}
\label{sec:related}

In \cite{Fazzinga.et.al:2018}, the authors studied the complexity of BAFs under deductive support as in \cite{Cayrol:Lagasquie-Schiex:2013}. 
Specifically, \citeauthor{Fazzinga.et.al:2018} analysed the verification problem \textbf{VER} 
and established that it is in $\P$ under admissible, stable, complete and grounded semantics, 
and in $\coNP$ under preferred semantics. 
We instead studied the complexity of bipolar ABA, 
and thus indirectly of BAFs
not only under deductive support, 
but also under other interpretations of support and with diverging semantics, 
as captured in bipolar ABA \cite{Cyras:Schulz:Toni:2017}. 
In addition, we analysed all the complexity problems standard in argumentation, 
namely \textbf{EX}, \textbf{NE}, \textbf{VER}, \textbf{CA}, \textbf{SA} and \textbf{EE}. 
To our knowledge, these problems have not been investigated for BAFs, except for the work of \citeauthor{Fazzinga.et.al:2018}. 
We restricted our study to the admissible, preferred and set-stable semantics of bipolar ABA used to capture various BAFs, 
but we will extend our analysis to other semantics in the future.

Complexity of ABA was investigated in \cite{Dimopoulos:Nebel:Toni:2002}. 
\citeauthor{Dimopoulos:Nebel:Toni:2002} studied general non-flat ABA with respect to the complexity of the derivation problem in the underlying deductive system of an ABA framework, 
as well as various instances of ABA, including the (flat) logic programming instance, called LP-ABA. 
Specifically, they established the generic upper bounds for \complexityProblem{VER}, 
as well as both upper bounds and instance-specific lower bounds for 
\complexityProblem{CA} and \complexityProblem{SA} under admissible, preferred and stable semantics. 
We note that \complexityProblem{DER} in LP-ABA belongs to $\P$ \cite{Dimopoulos:Nebel:Toni:2002}, and AFs can be mapped in \P-time into LP-ABA \cite{Toni:2012}. 
Thus, the results proven in this paper apply to LP-ABA as well. 
In particular, results provided in Section \ref{sec:complexityResults} complement the original work of \citeauthor{Dimopoulos:Nebel:Toni:2002} on LP-ABA by giving new lower bounds for  \complexityProblem{EX}, \complexityProblem{NE} and \complexityProblem{VER} problems.

The Tweety libraries \cite{Thimm:2017} provide implementations of various argumentation formalisms including AFs and ABA. Tweety can enumerate extensions of ABA frameworks under five semantics, including admissible, preferred and stable, 
but not set-stable semantics. 
Tweety essentially takes a brute force approach. 
For example, to compute the preferred extensions, it first generates all possible sets of assumptions and checks which ones are admissible. It then iterates through all these and checks which are maximal. This is very slow, as witnessed e.g.\ in a framework with ten assumptions, where Tweety takes more than 5 minutes to calculate extensions. In contrast, we showed our algorithms to be efficient in situations with hundreds of assumptions (on the same hardware). 

In \cite{DBLP:journals/argcom/EglyGW10}, \citeauthor{DBLP:journals/argcom/EglyGW10} provide an implementation of deductive BAFs \cite{Cayrol:Lagasquie-Schiex:2013}, but not of other approaches to BAFs. Their system reduces the problems to instances of answer set programming whereas ours works by directly calculating extensions.
There are also other implementations of structured argumentation formalisms (see \cite{Cerutti:Gaggl:Thimm:Wallner:2017} for a recent survey), 
and those relevant to ABA (e.g.\ 
\cite{Kakas:Moraitis:2003,Garcia:Simari:2014,Gordon:Walton:2016}) 
are reviewed in \cite{Bao:Cyras:Toni:2017}. 
Except for the Tweety libraries discussed above, to the best of our knowledge no other implementations of non-flat ABA in general, 
or bipolar ABA in particular, exist.


\section{Conclusions and Future Work}
\label{sec:conclusions}

In this paper, we established the computational complexity of six problems, 
namely (non-empty) existence, verification, (credulous and sceptical) acceptance and enumeration, 
for bipolar Assumption-Based Argumentation (ABA) under the admissible, preferred and set-stable semantics. 
Our results carry over to various Bipolar Argumentation Frameworks (BAFs) that are instances of bipolar ABA. 
We also provided novel algorithms for extension enumeration, consequently addressing the remaining problems, for bipolar ABA. 
Using these algorithms, we gave an implementation of bipolar ABA and various BAFs, 
and showed that it scales well. 
We have therefore provided solid theoretical foundations and realised an implementation underlying the practical deployment of bipolar argumentation. 

In the future, we plan on extending our analysis to generalisations of bipolar ABA. We will explore whether empowering these frameworks with new capabilities, such as support for factual rules or rules with multiple elements in their body, will lead to an increase in complexity. Moreover, we plan to extend our labelling algorithms to work for all ABA frameworks. Such algorithms will find use in an even wider range of practical scenarios than those described in this paper, due to the higher expressive power of generic ABA.

\subsubsection*{Acknowledgements}

The authors were supported by the EPSRC project 
\textbf{EP/P029558/1} ROAD2H: 
\emph{Resource Optimisation, Argumentation, Decision Support and Knowledge Transfer to Create Value via Learning Health Systems.}

\noindent
\textbf{Data access statement}: 
All data created during this research is available at \url{github.com/AminKaram/BipolarABASolver}. 
For more information please contact Amin Karamlou at \href{mak514@ic.ac.uk}{mak514@ic.ac.uk}.

\bibliographystyle{ACM-Reference-Format}  
\balance
\bibliography{references}

\end{document}